\documentclass[letterpaper]{article}
\usepackage{aaai}
\usepackage{times}
\usepackage{helvet}
\usepackage{courier}
\usepackage{amsthm}
\usepackage{algorithm}
\usepackage{algorithmic}
\usepackage{multirow}
\usepackage{graphics}
\usepackage{amsmath}
\usepackage{subfigure}

\newtheorem{definition}{Definition}
\newtheorem{lemma}{Lemma}

\newtheorem{corollary}{Corollary}
\newtheorem{theorem}{Theorem}

\begin{document}
%
\title{A Formal Analysis of Required Cooperation in Multi-agent Planning}
\author{Yu Zhang and Subbarao Kambhampati\\
School of Computing and Informatics\\
Arizona State University\\
Tempe, Arizona 85281 USA\\
\{yzhan442,rao\}@asu.edu\\
}
\maketitle

\begin{abstract}
\begin{quote}
Research on multi-agent planning has been popular in recent years. While previous research has been motivated by the understanding that, through cooperation, multi-agent systems can achieve tasks that are unachievable by single-agent systems, there are no formal characterizations of situations where cooperation is required to achieve a goal, thus warranting the application of multi-agent systems. In this paper, we provide such a formal discussion from the planning aspect. We first show that determining whether there is required cooperation (RC) is intractable is general. Then, by dividing the problems that require cooperation (referred to as RC problems) into two classes -- problems with heterogeneous and homogeneous agents, we aim to identify all the conditions that can cause RC in these two classes. We establish that when none of these identified conditions hold, the problem is single-agent solvable. Furthermore, with a few assumptions, we provide an upper bound on the minimum number of agents required for RC problems with homogeneous agents. This study not only provides new insights into multi-agent planning, but also has many applications. For example, in human-robot teaming, when a robot cannot achieve a task, it may be due to RC. In such cases, the human teammate should be informed and, consequently, coordinate with other available robots for a solution. 
\end{quote}
\end{abstract}

\section{Introduction}
\label{sec:intro}

A multi-agent planning (MAP) problem differs from a single agent planning (SAP) problem in that more than one agent is used in planning. While a (non-temporal) MAP problem can be compiled into a SAP problem by considering agents as resources, the search space grows exponentially with the number of such resources. Given that a SAP problem with a single such resource is in general PSPACE-complete \cite{Bylander1991}, running a single planner to solve MAP is inefficient. Hence, previous research has generally agreed that agents should be considered as separate entities for planning, and thus has been mainly concentrated on how to explore the interactions between the agents (i.e., loosely-coupled vs. tightly-coupled) to reduce the search space, and how to perform the search more efficiently in a distributed fashion.

However, there has been little discussion on whether multiple agents are required for a planning problem in the first place. If a single agent is sufficient, solving the problem with multiple agents becomes an efficiency matter, e.g., shortening the makespan of the plan. Problems of this nature can be solved in two separate steps: planning with a single agent and optimizing with multiple agents. In such a way, the difficulty of finding a solution may potentially be reduced.

In this paper, we aim to answer the following questions: $1)$ Given a problem with a set of agents, what are the conditions that make cooperation between multiple agents {\em required} to solve the problem; $2)$ How to determine the minimum number of agents required for the problem. We show that providing the exact answers is intractable. Instead, we attempt to provide approximate answers. To facilitate our analysis, we first divide MAP problems into two classes -- MAP problems with heterogeneous agents, and MAP problems with homogeneous agents. Consequently, the MAP problems that {\em require cooperation} (referred to as RC problems) are also divided into two classes -- type-$1$ RC (RC with heterogeneous agents) and type-$2$ RC (RC with homogeneous agents) problems. Figure \ref{fig:map-div} shows these divisions.

For the two classes of RC problems, we aim to identify all the conditions that can cause RC. Figure \ref{fig:rc-div} presents these conditions and their relationships to the two classes of RC problems. We establish that at least one of these conditions must be present in order to have RC. Furthermore, we show that most of the problems in common planning domains belong to type-$1$ RC, which is identified by three conditions in the problem formulation that define the heterogeneity of agents; most of the problems in type-$1$ RC can be solved by a {\em super agent}. For type-$2$ RC, we show that RC is only caused when the state space is not {\em traversable} or when there are {\em causal loops} in the causal graph. We provide upper bounds for the answer of the second question for type-$2$ RC problems, based on different relaxations of the conditions that cause RC, which are associated with, for example, how certain causal loops can be broken in the causal graph.

\begin{figure}
\centering
{
    \includegraphics{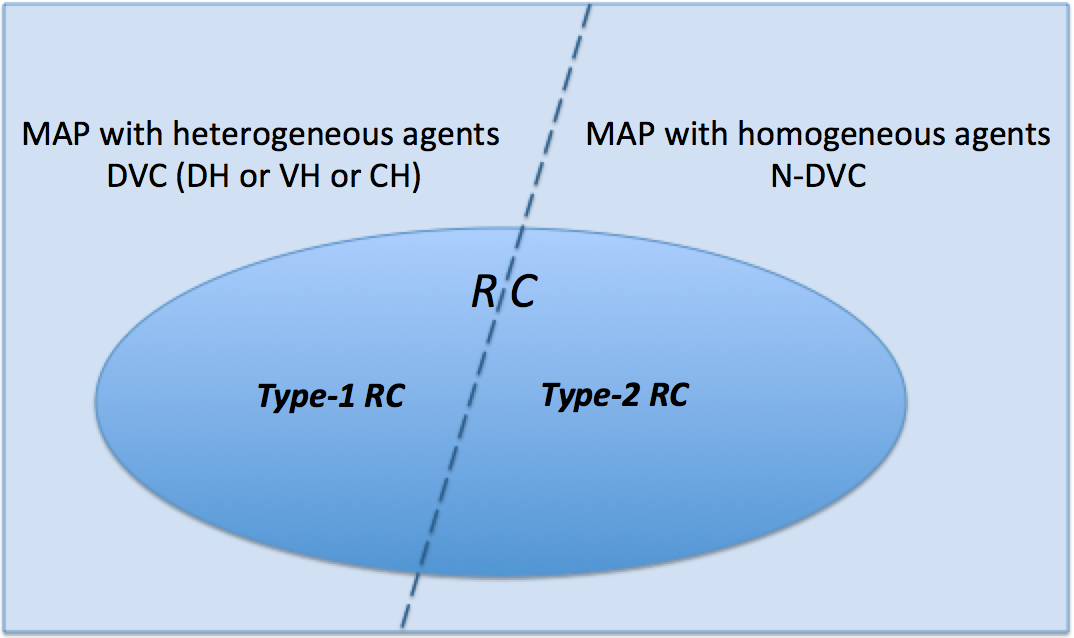}
}
\caption{Division of MAP problems into MAP with heterogeneous and homogeneous agents. Consequently, RC problems are also divided into two classes: type-$1$ RC involves RC problems with heterogeneous agents and type-$2$ RC involves RC problems with homogeneous agents.}
\label{fig:map-div}
\end{figure}

\begin{figure}
\centering
{
    \includegraphics{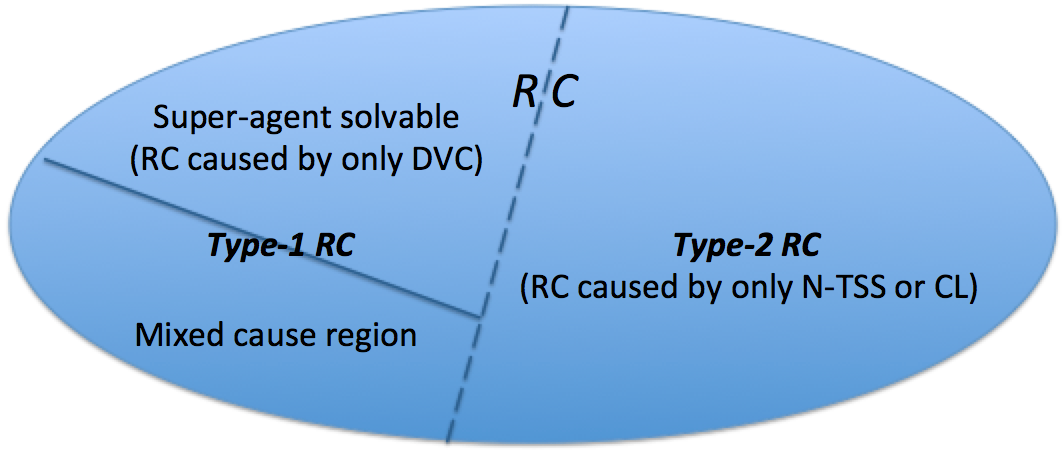}
}
\caption{Causes of required cooperation in RC problems.}
\label{fig:rc-div}
\end{figure}

The answers to these questions not only enrich our fundamental understanding of MAP, but also have many applications. For example, in a human robot teaming scenario, a human may be remotely working with multiple robots. When a robot is assigned a task that it cannot achieve, it is useful to determine whether the failure is due to the fact that the task is simply unachievable or the task requires more than one robot. In the latter case, it is useful then to determine how many extra robots must be sent to help. The answers can also be applied to multi-robot systems, and are useful in general to any multi-agent systems in which the team compositions can dynamically change (e.g., when the team must be divided to solve different problems).

The rest of the paper is organized as follows. After a review of the related literature, we start the discussion of required cooperation for MAP, in which we answer the above questions in an orderly fashion. We conclude afterward.

\section{Related Work}
\label{sec:related}

One of the earlier works on MAP is the PGP framework by \cite{Durfee91,Decker92}. Recently, the MAP problem has started to receive an increasing amount of attention. Most of these recent research works consider agents separately for planning, and have been concentrated on how to explore the structure of agent interactions to reduce the search space, as well as solving the problem in a distributed fashion. \cite{Nissim2010} provide a search method by compiling MAP into a constraint satisfaction problem (CSP), and then using a distributed CSP framework to solve it. The MAP formulation is based on an extension of the STRIPS language called MA-STRIPS \cite{brafman2008}. In MA-STRIPS, actions are categorized into public and private actions. Public actions can influence other agents while private actions cannot. In this way, it is shown by \cite{brafman2008} that the search complexity of MAP is exponential in the tree-width of the agent interaction graph. Due to the poor performance of DisCSP based approaches, \cite{Nissim2012} apply the $A^*$ search algorithm in a distributed manner, which represents one of the state-of-art MAP solvers. \cite{torreno2012} propose a POP-based distributed planning framework for MAP, which uses a cooperative refinement planning technique that can handle planning with any level of coupling between the agents. Each agent at any step proposes a refinement step to improve the current group plan. Their approach does not assume complete information. A similar paradigm is taken by \cite{kvarnstrom11}. An iterative best-response planning and plan improvement technique using standard SAP algorithms is provided by \cite{jonsson2011}, which considers the previous singe agent plans as constraints to be satisfied while the following agents perform planning.

Given a problem, all of these MAP approaches solve it using the given set of agents, without first asking whether multiple agents are really required, let alone what is the minimum number of agents required. Answers to these questions not only separate MAP from SAP in a fundamental way, but also have real world applications when the team compositions can dynamically change. In this paper, we analyze these questions using the SAS$^+$ formalism \cite{backstrom96} with {\em causal graph} \cite{Knoblock94,helmert06}, which is often discussed in the context of factored planning \cite{bacchus93,amir03,brafman06,brafman2013}. The causal graph captures the interaction between different variables; intuitively, it can also capture the interactions between agents since agents affect each other through these variables. In fact, \cite{brafman2013} mention the causal graph's relation to the agent interaction graph when each variable is associated with a single agent.

\section{Multi-agent Planning (MAP)}
\label{sec:ma-planning}

In this paper, we start the analysis of RC in the simplest scenarios -- with instantaneous actions and sequential execution. The possibility of RC can only increase when we extend the model to the temporal domain, in which concurrent or synchronous actions must be considered. We develop our analysis of required cooperation for MAP based on the SAS$^+$ formalism \cite{backstrom96}.

\subsection{Background}

\begin{definition}
A SAS$^+$ problem is given by a tuple $P = \langle V, A, I, G \rangle$, where:
	\begin{itemize}
		\item $V = \{v_1, ..., v_n\}$ is a set of state variables. Each variable $v_i \in V$ is associated with its domain $D(v_i)$, which is used to define an extended domain $D(v_i)^+ = D(v_i) \cup u$, where $u$ denotes the undefined value. The state space is defined as $S^+_V = D(v_1)^+ \times ... \times D(v_n)^+$; $s[v_i]$ denotes the value of the variable $v_i$ in a state $s \in S^+_V$.
		\item $A = \{a_1, ..., a_m\}$ is a finite set of actions. Each action $a_j$ is a tuple $\langle pre(a_j), post(a_j), prv(a_j) \rangle$, where $pre(a_j), post(a_j), prv(a_j) \subseteq S^+_V$ are the preconditions, postconditions and prevail conditions of $a_j$, respectively. We also use $pre(a_j)[v_i], post(a_j)[v_i], prv(a_j)[v_i]$ to denote the corresponding values of $v_i$. 
		\item $I$ and $G$ denote the initial and goal state, respectively.
	\end{itemize}
\label{def:sas}
\end{definition}

A plan in SAS$^+$ is often defined to be a total-order plan:

\begin{definition}
A plan $\pi$ in SAS$^+$ is a sequence of actions $\pi = \langle a_1,..., a_l \rangle$.
\label{def:sas-plan}
\end{definition}

Given two states $s_1$, $s_2 \in S^+_V$, $(s_1 \oplus s_2)$ denotes that $s_1$ is updated by $s_2$, and is subject to the following for all $v_i \in V$:

\begin{equation}
(s_1 \oplus s_2)[v_i] = \left\{ 
  \begin{array}{l l}
    s_2[v_i] &  \text{if $s_2[v_i] \neq u$,}\\
    s_1[v_i] &  \text{otherwise.}
  \end{array} \right.
\end{equation}

Given a variable with two values $x, y$ in which one of them is $u$, $x \sqcup y$ is defined to be the other value. $\sqcup$ can be extended to two states $s_1$ and $s_2$, such that $s_1 \sqcup s_2[vi] = s_1[v_i] \sqcup s_2[v_i]$ for all $v_i \in V$. $s_1 \sqsubseteq s_2$ if and only if $\forall v_i \in V, s_1[v_i] = u$ or $s_1[v_i] = s_2[v_i]$. The state resulting from executing a plan $\pi$ can then be defined recursively using a $re$ operator as follows:

\begin{equation}
re(s, \langle \pi;o \rangle) = \left\{ 
  \begin{array}{l l}
    re(s, \langle \pi \rangle) \oplus post(o) \\
    \text{if $pre(o) \sqcup prv(o) \sqsubseteq re(s, \langle \pi \rangle)$,}\\
    s  \quad \text{otherwise.}
  \end{array} \right.
\end{equation}
in which $re(s, \langle \rangle) = s$, $o$ is an action, and $;$ is the concatenation operator.

\subsection{Extension to MAP}

To extend the previous formalism to MAP without losing generality, we minimally modify the definitions.

\begin{definition}
A SAS$^+$ MAP problem is given by a tuple $\Pi = \langle V, \Phi, I, G \rangle$, where:
	\begin{itemize}
		\item $\Phi = \{\phi_g\}$ is the set of agents; each agent $\phi_g$ is associated with a set of actions $A(\phi_g)$. 
	\end{itemize}
\end{definition}

\begin{definition}
A plan $\pi_{MAP}$ in MAP is a sequence of agent-action pairs $\pi_{MAP} = \langle (a_1, \phi(a_1)),..., (a_L, \phi(a_L)) \rangle$, in which $\phi(a)$ returns the agent for the action $a$ and $L$ is the length of the plan.
\end{definition}

We do not need to consider concurrency or synchronization given that actions are assumed to be instantaneous.

\section{Required Cooperation for MAP}

Next, we formally define the notion of {\em required cooperation} and other useful terms that are used in the following analyses. We assume throughout the paper that more than one agent is considered (i.e., $|\Phi| > 1$).

\subsection{Required Cooperation}

\begin{definition}[$k$-agent Solvable]
Given a MAP problem $P = \langle V, \Phi, I, G \rangle$ ($|\Phi| \geq k$), the problem is $k$-agent solvable if $\exists \Phi_k \subseteq \Phi$ ($|\Phi_k| = k$), such that $\langle V, \Phi_k, I, G \rangle$ is solvable.
\label{def:k-agent}
\end{definition}

\begin{definition}[Required Cooperation (RC)]
Given a solvable MAP problem $P = \langle V, \Phi, I, G \rangle$, there is required cooperation if it is not $1$-agent solvable.
\label{def:rc}
\end{definition}

In other words, given a solvable MAP problem that satisfies RC, any plan must involve more than one agent.

\begin{lemma}
Given a solvable MAP problem $P = \langle V, \Phi, I, G \rangle$, determining whether it satisfies RC is PSPACE-complete.
\label{lem:rc}
\end{lemma}

\begin{proof}
First, it is not difficult to show that the RC decision problem belongs to PSPACE, since we only need to verify that $P = \langle V, \phi, I, G \rangle$ is unsolvable for all $\phi \in \Phi$, given that the initial problem is known to be solvable. Then, we complete the proof by reducing from the PLANSAT problem, which is PSPACE-complete in general \cite{Bylander1991}. Given a PLANSAT problem (with a single agent), the idea is that we can introduce a second agent with only one action. This action directly achieves the goal but requires an action (with all preconditions satisfied in the initial state) of the initial agent to provide a precondition that is not initially satisfied. We know that this constructed MAP problem is solvable. If the algorithm for the RC decision problem returns that cooperation is required for this MAP problem, we know that the original PLANSAT problem is unsolvable; otherwise, it is solvable. 
\end{proof}

\begin{definition}[Minimally $k$-agent Solvable]
Given a solvable MAP problem $P = \langle V, \Phi, I, G \rangle$ ($|\Phi| \geq k$), it is minimally $k$-agent solvable if it is $k$-agent solvable, and not $(k$$-$$1)$-agent solvable.
\label{def:m-k-agent}
\end{definition}

\begin{corollary}
Given a solvable MAP problem $P = \langle V, \Phi, I, G \rangle$, determining the minimally solvable $k$ ($k \leq |\Phi|$) is PSPACE-complete.
\label{cor:krc}
\end{corollary}

Although directly querying for RC is intractable, we aim to identify all the conditions (which can be quickly checked) that can cause RC. We first define a few terms that are used in the following discussions.

We note that the reference of agent is explicit in the action (i.e., ground operator) parameters. Although actions are unique for each agent, two different agents may be capable of executing actions that are instantiated from the same operator, with all other parameters being identical. To identify such cases, we introduce the notion of {\em action signature}.

\begin{definition}[Action Signature (AS)]
An action signature is an action with the reference of the executing agent replaced by a global $EX$-$AG$ symbol. 
\label{def:sig}
\end{definition}

For example, an action signature in the IPC logistics domain is $drive(EX$-$AG, pgh$-$po, pgh$-$airport)$. $EX$-$AG$ is a global symbol to denote the executing agent, which is not used to distinguish between action signatures. We denote the set of action signatures for $\phi \in \Phi$ as $AS(\phi)$, which specifies the {\em capabilities} of $\phi$. Furthermore, we define the notion of {\em agent variable}.

\begin{definition}[Agent Variable (Agent Fluent)]
A variable (fluent) is an agent variable (fluent) if it is associated with the reference of an agent.
\label{def:ag-atom}
\end{definition}

Agent variables are used to specify agent state. For example, $location(truck$-$pgh)$ is an agent variable since it is associated with an agent $truck$-$pgh$. We use $V_{\phi} \subseteq V$ to denote the set of agent variables that are associated with $\phi$ (i.e., variables that are present in the initial state or actions of $\phi$).

Following this notation, we can rewrite a MAP problem as $P = \langle V_o \cup V_{\Phi}, \Phi, I_o \cup I_{\Phi}, G_o \cup G_{\Phi} \rangle$, in which $V_{\Phi} = \{V_{\phi}\}$, $I_{\Phi} = \{I_{\phi}\}$, $G_{\Phi} = \{G_{\phi}\}$, $I_{\phi} = I \cap V_{\phi}$ and $G_{\phi} = G \cap V_{\phi}$. $V_o$ denotes the set of non-agent variables; $I_o$ and $G_o$ are the set of non-agent variables in $I$ and $G$, respectively. In this paper, we assume that agents can only interact with each other through non-agent variables (i.e., $V_o$). In other words, agent variables contain one and only one reference of agent. As a result, we have $V_{\phi} \cap V_{\phi'} \equiv \emptyset$ ($\phi \neq \phi'$). It seems to be possible to compile away exceptions by breaking agent variables (with more than one reference of agent) into multiple variables and introducing non-agent variables to correlate them.

\begin{definition}[Variable (Fluent) Signature (VS)]
Given an agent variable (fluent), its variable (fluent) signature is the variable (fluent) with the reference of agent replaced by $EX$-$AG$.
\label{def:at-sig}
\end{definition}

For example, $location(truck$-$pgh)$ is an agent variable for $truck$-$pgh$ and its variable signature is $location(EX$-$AG)$. We denote the set of VSs for $V_\phi$ as $VS(\phi)$, and use $VS$ as an operator so that $VS(v)$ returns the VS of a variable $v$; this operator returns any non-agent variable unchanged.

\section{Classes of RC}

In the following discussion, we assume that the specification of goal (i.e., $G$) in the MAP problems does not involve agent variables (i.e., $G \cap V_{\phi} = \emptyset$ or $G_{\phi} = \emptyset$), since we are mostly interested in how to reach the desired world state (i.e., specified in terms of $V_o$). As aforementioned, we divide RC problems into two classes as shown in Figure \ref{fig:map-div}. Type-$1$ RC involves problems with heterogeneous agents; type-$2$ RC involves problems with homogeneous agents. Next, we formally define each class and discuss the causes of RC. Throughout this paper, when we denote a condition as X, the negated condition is denoted as N-X.

\subsection{Type-$1$ RC (RC with Heterogeneous Agents)}
\label{sec:type-1}

Given a MAP problem $P = \langle V, \Phi, I, G \rangle$, the heterogeneity of agents can be characterized by the following conditions:

\begin{itemize}
	\item {\em Domain Heterogeneity (DH)}: $\exists v \in V_{\phi}$ and $D(v) \setminus D(V') \neq \emptyset$, in which $V' = \{v' | v' \in V_{\phi'} (\phi' \neq \phi)$ and $VS(v) = VS(v')\}$.
	\item {\em Variable Heterogeneity (VH)}: $VS(\phi) \setminus VS(\Phi \setminus \phi) \neq \emptyset$.
	\item {\em Capability Heterogeneity (CH)}: $AS(\phi) \setminus AS(\Phi \setminus \phi) \neq \emptyset$.
\end{itemize}

\begin{definition}[Type-1 RC]
An RC problem belongs to type-$1$ RC if at least one of DH, VH and CH is satisfied for an agent.
\label{def:type1}
\end{definition}

The condition that requires at least one of DH, VH and CH to be satisfied is denoted as DVC in Figure \ref{fig:map-div}. It is worth noting that when considering certain objects (e.g., truck and plane in the logistics domain) as agents rather than as resources, most of the RC problems in the IPC domains belong to type-$1$ RC.

\subsection{Causes of RC in Type-$1$}
\label{sec:type-1-cause}

The most obvious condition for RC in type-$1$ RC problems is due to the heterogeneity of agents. In the logistics domain, for example, if any truck agent can only stay in one city, the domains of the location variable for different truck agents are different (DH). When there are packages that must be transferred between different locations within cities, at least one truck agent for each city is required (hence RC). 
In the rover domain, a rover that is not equipped with a camera sensor would not be associated with the agent variable $equipped\_for\_imaging$. When we need both $equipped\_for\_imaging$ and $equipped\_for\_rock\_analysis$, and no rovers are equipped with the sensors for both (VH), we have RC. Note that VH does not specify any requirement on the variable value (i.e., the state); however, when the domain of a variable contains only a single value, e.g., $equipped\_for\_imaging$, we assume in this paper that this variable is always defined in a positive manner, e.g., expressing cans instead of cannots.  
In the logistics domain, given that the truck agent cannot fly (CH), when a package must be delivered from a city to a non-airport location of another city, at least a truck and a plane are required. Note that DH, VH and CH are closely correlated.

However, note that $1)$ the presence of DVC in a solvable MAP problem does not always cause RC, as shown in Figure \ref{fig:map-div}; $2)$ the presence of DVC in a type-$1$ RC problem is not always the cause of RC, as shown in Figure \ref{fig:rc-div}.

As an example for $1)$, when there is only one package to be delivered from one location to another within the same city, there is no need for a plane agent, even though we can create a non-RC MAP problem with a plane and a truck agent that satisfies CH (thus DVC).

As an example for $2)$, for navigating in a grid world, the traversability of the world for all mobile agents can be restricted based on edge connections, i.e., $connected(a, b)$, in which $a$ and $b$ are vertices in the grid.
Suppose that we have two packages to be delivered to locations $b$ and $c$, respectively, which are both initially at $a$. There are two truck agents at $a$ that can be used for delivery. However, the paths from $a$ to both $b$ and $c$ are one-way only (i.e., $connected(a, b) = true$ and $connected(b, a) = false$). Even if one of the truck agents uses gas and the other one uses diesel, thus satisfying DVC, it is clear that RC in this problem is not caused by the heterogeneity of agents.

Type-$1$ RC problems in which RC is caused by only DVC can be solved by a {\em super agent} (defined below), which is an agent that combines all the domain values, variable signatures and capabilities (i.e., action signatures) of the other agents. We refer to the subset of type-$1$ RC problems that can be solved by a super agent as {\it super-agent solvable}, as shown in Figure \ref{fig:rc-div}.

\begin{definition}[Super Agent]
A super agent is an agent $\phi^*$ that satisfies:
\begin{itemize}
	\item $\forall v \in V_{\Phi}$, $\exists v^* \in V_{\phi^*}$, $D(v^*) = D(V)$, in which $V = \{v | v \in V_{\Phi}$ and $VS(v^*) = VS(v)\}$.
	\item $VS(\phi^*) = VS(\Phi)$.
	\item $AS(\phi^*) = AS(\Phi)$.
\end{itemize}
\end{definition}

It is not difficult to see that most problems in the IPC domains are also super-agent solvable. For example, when we have a truck-plane agent in the logistics domain that can both fly (between airports of different cities) and drive (between locations in the same cities), or when we have a rover that is equipped with all sensors and can traverse all waypoints in the rover domain.

From Figure \ref{fig:rc-div}, one may have already noticed that the conditions that cause RC in type-$2$ problems may also cause RC in type-$1$ problems (i.e., indicated by the {\it mixed cause region} in Figure \ref{fig:rc-div}). For example, the aforementioned example for navigating in a grid world demonstrates that the initial states (specified in terms of the values for variables) of different agents may cause RC in type-$1$ problems. Note that the initial states of different agents cannot be combined as for domain values, variable signatures and capabilities in a super agent construction; however, the special cases when the domains of variables contain only a single value (when we discussed VH in {\textbf{Causes of RC in Type-$1$}}) can also be considered as cases when RC is caused by the initial state.

\subsection{Type-$2$ RC (RC with Homogeneous Agents)}
\label{sec:type-2}

Type-$2$ RC involves homogeneous agents:

\begin{definition}[Type-$2$ RC]
An RC problem belongs to type-2 RC if it satisfies N-DVC (for all agents). 
\label{def:src}
\end{definition}

Definition \ref{def:src} states that an RC problem belongs to type-$2$ RC when all the agents are homogeneous.

\subsection{Type-$2$ RC Caused by Traversability}
\label{sec:first-cause-2}

One condition that causes RC in type-$2$ RC problems is the {\it traversability} of the state space of variables, which is related to the initial states of the agents and the world, as we previously discussed. Since the traversability is associated with the evolution of variable values, we use causal graphs to perform the analysis.

\begin{definition}[Causal Graph]
Given a MAP problem $P = \langle V, \Phi, I, G \rangle$, the causal graph $G$ is a graph with directed and undirected edges over the nodes $V$. For two nodes $v$ and $v'$ ($v \neq v'$), a directed edge $v \rightarrow v'$ is introduced if there exists an action that updates $v'$ while having a prevail condition associated with $v$. An undirected edge $v - v'$ is introduced if there exists an action that updates both.
\end{definition}

A typical example of a causal graph for an individual agent is presented in Figure \ref{fig:no-loop}. For type-$2$ RC study, since the agents are homogeneous, the causal graphs for all agents are the same. Hence, we can use agent VSs to replace agent variables; we refer to this modified causal graph for a single agent in a type-$2$ RC problem as an {\em individual causal graph signature} (ICGS). Next, we define the notions of {\em closures} and {\em traversable state space}.

\begin{figure}
\centering
{
    \includegraphics{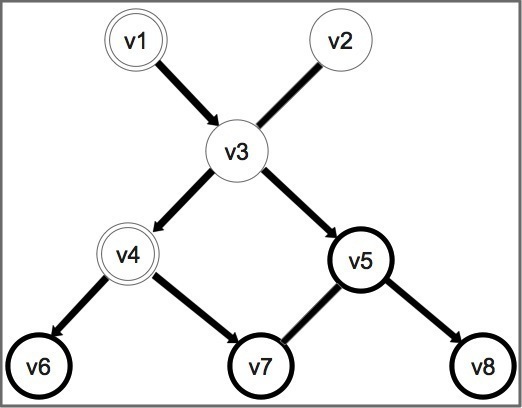}	
}
\caption{Example of a causal graph (ICGS). Variables in goal $G$ are shown as bold-circle nodes and agent VSs are shown as double-circle nodes.}
\label{fig:no-loop}
\end{figure}

\begin{definition}[Inner and Outer Closures (IC and OC)]
An inner closure (IC) in an ICGS is any set of variables for which no other variables are connected to them with undirected edges; an outer closure (OC) of an IC is the set of nodes that have directed edges going into nodes in the IC.  
\end{definition}

In Figure \ref{fig:no-loop}, $\{v_2, v_3\}$ and $\{v_4\}$ are examples of ICs. The OC of $\{v_2, v_3\}$ is $\{v_1\}$ and the OC of $\{v_4\}$ is $\{v_3\}$.

\begin{definition}[Traversable State Space (TSS)]
An IC has a traversable state space if and only if: given any two states of this IC, denoted by $s$ and $s'$, there exists a plan that connects them, assuming that the state of the OC of this IC can be changed freely within its state space. 
\end{definition}

In other words, an IC has a TSS if the traversal of its state space is only dependent on the variables in its OC; this also means that when the OC of an IC is empty, the state of the IC can change freely. Note that static variables in the OC of an IC can assume values that do not influence the traversability. For example, the variables that are used to specify the connectivity of vertices in a grid, e.g., $connected(a, b)$, can be assigned to be $true$ or $false$; although the variables that are assigned to be $true$ cannot change their values to be $false$, they do not influence the traversability of the grid world. In such cases, the associated ICs are still considered to have a TSS. An ICGS in which all ICs have TSSs is referred to as being traversable.

\subsection{Type-$2$ RC Caused by Causal Loops}
\label{sec:second-cause-2}

However, even a solvable MAP problem that satisfies N-DVC for all agents while having a traversable ICGS can still satisfy RC. An example is presented below.

The goal of this problem is to steal a diamond from a room, in which the diamond is secured, and place it in another room. The diamond is protected by a stealth detection system. If the diamond is taken, the system locks the door of the room in which the diamond is kept, so that the insiders cannot exit. There is a switch to override the detection system but it is located outside of the room. This problem is modeled as above, in which the value is immediately specified after each variable. It is not difficult to see that the above problem cannot be solved with a single agent.

\makebox{} \par
\framebox
{
	\begin{minipage}{7.6cm}
    	{\bf Initial State}: 
	\begin{quote}
	$location(agent1)$ $room1$ \newline
	$location(agent2)$ $room1$ \newline
	$location(diamond1)$ $room1$ \newline
	$doorLocked(room1)$ $false$ \newline
	$location(switch1)$ $room2$
	\end{quote}
	
	{\bf Goal State}:
	\begin{quote}
	$location(diamond1)$ $room2$ 
	\end{quote}

   	\end{minipage}
}
\makebox{} \par

\makebox{} \par
\framebox
{
	\begin{minipage}{7.6cm}
	{\bf Operators}: \newline
    	$WalkThrough(agent, door, fromRoom, toRoom)$: 
	\begin{quote}
	$prv$: $doorLocked(door)$ $false$ \newline
	$pre$: $location(agent)$ $fromRoom$ \newline
	$post$: $location(agent)$ $toRoom$
	\end{quote}
	
	$Steal(agent, diamond, room, door)$:
	\begin{quote}
	$prv$: $location(agent)$ $room$ \newline
	$pre$: $doorLocked(door)$ $u$\newline
	$pre$: $location(diamond)$ $room$ \newline
	$post$: $doorLocked(door)$ $true$ \newline
	$post$: $location(diamond)$ $agent$
	\end{quote}
	
	$Switch(agent, switch, room, door)$:
	\begin{quote}
	$prv$: $location(switch)$ $room$ \newline
	$prv$: $location(agent)$ $room$ \newline
	$pre$: $doorLocked(door)$ $u$ \newline
	$post$: $doorLocked(door)$ $false$
	\end{quote}
	
	$Place(agent, diamond, room)$:
	\begin{quote}
	$prv$: $location(agent)$ $room$ \newline
	$pre$: $location(diamond)$ $agent$ \newline
	$post$: $location(diamond)$ $room$ 
	\end{quote}	

   	\end{minipage}
}
\makebox{} \par

Again, we construct the ICGS for this type-$2$ RC example, as shown in Figure \ref{fig:single-agent}. One key observation is that a single agent cannot address this problem due to the fact that $WalkThrough$ with the diamond to $room2$ requires $doorLocked(door1) = false$, which is violated by the $Steal$ action to obtain the diamond in the first place. This is clearly related to the loops in Figure \ref{fig:single-agent}. In particular, we define the notion of {\em causal loops}.

\begin{definition}[Causal Loop (CL)]
A causal loop in the ICGS is a directed loop that contains at least one directed edge. 
\label{def:cloop}
\end{definition}

Note that undirected edges can be considered as edges in either direction but at least one directed edge must be present in a causal loop.

\subsection{Gap between MAP and Single Agent Planning}

\begin{figure}
\centering
{
    \includegraphics{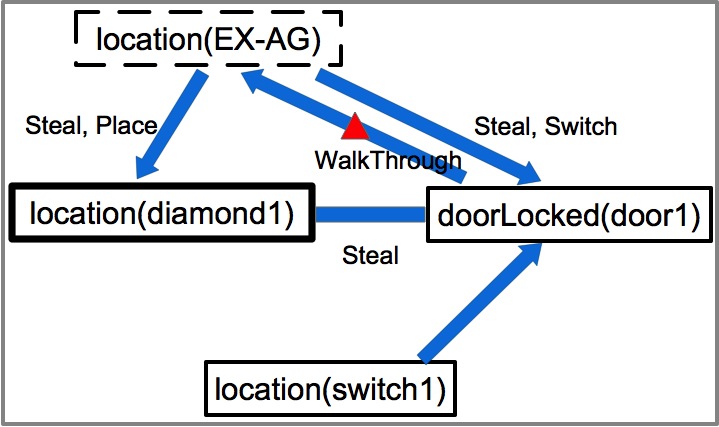}
}
\caption{ICGS for the diamond example that illustrates the second condition that causes RC in type-$2$ RC problems. Actions (without parameters) are labeled along with their corresponding edges. The variables in $G$ are shown as bold-box nodes and agent VSs are shown as dashed-box nodes.}
\label{fig:single-agent}
\end{figure}

We now establish in the following theorem that when none of the previously discussed conditions (for both type-$1$ and type-$2$ RC) hold in a MAP problem, this problem can be solved by a single agent.

\begin{theorem}
Given a solvable MAP problem that satisfies N-DVC for all agents, and for which the ICGS is traversable and contains no causal loops, any single agent can also achieve the goal.  
\label{thm:single}
\end{theorem}

\begin{proof}
Given no causal loops, the directed edges in the ICGS divides the variables into levels, in which: $1)$ variables at each level do not appear in other levels; $2)$ higher level variables are connected to lower level variables with only directed edges going from higher levels to lower levels; $3)$ variables within each level are either not connected or connected with undirected edges. For example, the variables in Figure \ref{fig:no-loop} are divided into the following levels (from high to low): $\{v_1\}$, $\{v_2, v_3\}$, $\{v_4\}$, $\{v_5, v_7\}$, $\{v_6, v_8\}$. Note that this division is not unique.

Next, we prove the result by induction based on the level. Suppose that the ICGS has $k$ levels and we have the following holds: given any trajectory of states for all variables, there exists a plan whose execution traces of states include this trajectory in the correct order.

When the ICGS has $k + 1$ levels: given any state $s$ for all variables from level $1$ to $k + 1$, we know from the assumption that the ICGS is traversable that there exists a plan that can update the variables at the $k + 1$ level from their current states to the corresponding states in $s$. This plan (denoted by $\pi$), meanwhile, requires the freedom to change the states of variables from level $1$ to $k$. Given the induction assumption, we know that we can update these variables to their required states in the correct order to satisfy $\pi$; furthermore, these updates (at level $k$ and above) also do not influence the variables at the $k + 1$ level (hence do not influence $\pi$). Once the states of the variables at the $k + 1$ level are updated to match those in $s$, we can then update variables at level $1$ to $k$ to match their states in $s$ accordingly. Using this process, we can incrementally build a plan whose execution traces of states contain any given trajectory of states for all the variables in the correct order.

Furthermore, the induction holds when there is only one level given that ICGS is traversable. Hence, the induction conclusion holds. The main conclusion directly follows. 
\end{proof}

\subsection{Towards an Upper Bound for Type-$2$ RC}

In this section, we investigate type-$2$ RC problem to obtain upper bounds on the $k$ (Definition \ref{def:m-k-agent}), based on different relaxations of the two conditions that cause RC in type-$2$ RC problems. We first relax the assumption regarding causal loops (CLs) and show that the relaxation process is associated with how certain CLs can be broken.

We notice that there are two kinds of CLs in ICGS. The first kind contains agent VSs while the second kind does not. Although we cannot break CLs for the second kind, it is possible to break CLs for the first kind. The motivation is that certain edges in these CLs can be removed when there is no need to update the associated agent VSs. In our diamond example, when there are two agents in $room1$ and $room2$, respectively, and they can stay where they are during the execution of the plan, there is no need to $WalkThrough$ and hence the associated edges can be removed to break the CLs. Figure \ref{fig:two-agent} shows this process. Based on this observation, we introduce the following lemma.

\begin{figure}
\centering
{
    \includegraphics{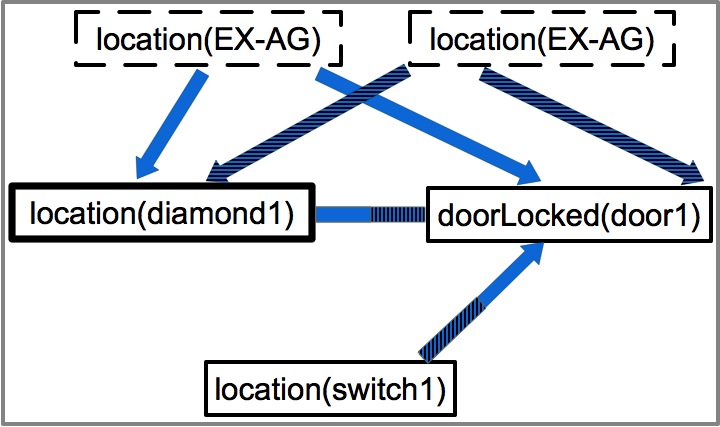}
}
\caption{Illustration of the process for breaking causal loops in the diamond example, in which the CLs are broken by removing the edge marked with a triangle in Figure \ref{fig:single-agent}. Two agent VSs are introduced to replace the original agent VS.}
\label{fig:two-agent}
\end{figure}

\begin{lemma}
Given a solvable MAP problem that satisfies N-DVC for all agents and for which the ICGS is traversable, if no CLs contain agent VSs and all the edges going in and out of agent VSs are directed, the minimum number of agents required is upper bounded by $\times_{v \in CR(\Phi)} |D(v)|$, when assuming that the agents can choose their initial states, in which $CR(\Phi)$ is constructed as follows:
\begin{enumerate}
	\item add the set of agent VSs that are in the CLs into $CR(\Phi)$;
	\item add in an agent VS into $CR(\Phi)$ if there exists a directed edge that goes into it from any variable in $CR(\Phi)$; 
	\item iterate $2$ until no agent VSs can be added.
\end{enumerate}
\label{lem:loops}
\end{lemma}

\begin{proof}
Based on the previous discussions, we can remove edges that are connected to agent VSs to break loops. For each variable in $CR(\Phi)$, denoted by $v$, we introduce a set of variables $N = \{v_1, v_2, ..., v_{|D(v)|}\}$ to replace $v$. Any edges connecting to $v$ from other variables are duplicated on all variables in $N$, except for the edges that go into $v$. Each variable $v_i \in N$ has a domain with a single value; this value for each variable in $N$ is different and chosen from $D(v)$. Note that these new variables do not affect the traversability of the ICGS.

From Theorem \ref{thm:single}, we know that a virtual agent $\phi^+$ that can simultaneously assume all the states that are the different permutations of states for $CR(\Phi)$ can achieve the goal. We can simulate $\phi^+$ using $\times_{v \in CR(\Phi)} |D(v)|$ agents as follows. We choose the agent initial states according to the permutations of states for $CR(\Phi)$, while choosing the same states for all the other agent VSs according to $\phi^+$. Given a plan for $\phi^+$, we start from the first action. Given that all permutations of states for $CR(\Phi)$ are assumed by an agent, we can find an agent, denoted by $\phi$, that can execute this action: $1)$ If this action updates an agent VS in $CR(\Phi)$, we do not need to execute this action based on the following reasoning. Given that all edges going in and out of agent VSs are directed, we know that this action does not update $V_o$. (Otherwise, there must be an undirected edge connecting a variable in $V_o$ to this agent VS. Similarly, we also know that this action does not update more than one agent VS.). As a result, it does not influence the execution of the next action. $2)$ If this action updates an agent VS that is not in $CR(\Phi)$, we know that this action cannot have variables in $CR(\Phi)$ as preconditions or prevail conditions, since otherwise this agent VS would be included in $CR(\Phi)$ given its construction process. Hence, all the agents can execute the action to update this agent VS, given that all the agent VSs outside of $CR(\Phi)$ are always kept synchronized in the entire process (in order to simulate $\phi^+$). $3)$ Otherwise, this action must be updating only $V_o$ and we can execute the action on $\phi$.

Following the above process for all the actions in $\phi^+$'s plan to achieve the goal. Hence, the conclusion holds.
\end{proof}

Next, we investigate the relaxation of the traversability of the ICGS.

\begin{lemma}
Given a solvable MAP problem that satisfies N-DVC for all agents, if all the edges going in and out of agent VSs are directed, the minimum number of agents required is upper bounded by $\times_{v \in VS(\Phi)} |D(v)|$, when assuming that the agents can choose their initial states.
\label{lem:trav}
\end{lemma}

\begin{proof}
Given a valid plan $\pi_{MAP}$ for the problem, we can solve the problem using $\times_{v \in VS(\Phi)} |D(v)|$ agents as follows: first, we choose the agent initial states according to the permutations of state for $VS(\Phi)$.

The process is similar to that in Lemma \ref{lem:loops}. We start from the first action. Given that all permutations of $VS(\Phi)$ are assumed by an agent, we can find an agent, denoted by $\phi$, that can execute this action: if this action updates some agent VSs in $VS(\Phi)$, we do not need to execute this action; otherwise, the action must be updating only $V_o$ and we can execute the action on $\phi$.

Following the above process for all the actions in $\pi_{MAP}$ to achieve the goal. Hence, the conclusion holds. 
\end{proof}

Note that the bounds in Lemma \ref{lem:loops} and \ref{lem:trav} are upper bounds and the minimum number of agents actually required may be smaller. Nevertheless, for the simple scenario in our diamond example, the assumptions of both lemmas are satisfied and the bounds returned are $2$ for both, which happens to be exactly the $k$ in Definition \ref{def:m-k-agent}. In future work, we plan to investigate other relaxations and establish the tightness of these bounds.

\section{Conclusion}
\label{sec:con}

In this paper, we introduce the notion of required cooperation (RC), which answers two questions: $1)$ whether more than one agent is required for a solvable MAP problem, and $2)$ what is the minimum number of agents required for the problem. We show that the exact answers to these questions are difficult to provide. To facilitate our analysis, we first divide RC problems into two class -- type-$1$ RC involves heterogeneous agents and type-$2$ RC involves homogeneous agents. For the first question, we show that most of the problems in the common planning domains belong to type-$1$ RC; the set of type-$1$ RC problems in which RC is only caused by DVC can be solved with a super agent. For type-$2$ RC problems, we show that RC is caused when the state space is not traversable or when there are causal loops in the causal graph; we provide upper bounds for the answer of the second question, based on different relaxations of the conditions that cause RC in type-$2$ RC problems. These relaxations are associated with, for example, how certain causal loops can be broken in the causal graph.

\section*{Acknowledgement}
This research is supported in part by the ARO grant W911NF-13-1-0023,
and the ONR grants N00014-13-1-0176 and  N00014-13-1-0519.

\bibliography{paper.bib} 

\begin{thebibliography}{}

\bibitem[\protect\citeauthoryear{Amir and Engelhardt}{2003}]{amir03}
Amir, E., and Engelhardt, B.
\newblock 2003.
\newblock Factored planning.
\newblock In {\em Proceedings of the 18th International Joint Conferences on
  Artificial Intelligence},  929--935.

\bibitem[\protect\citeauthoryear{Bacchus and Yang}{1993}]{bacchus93}
Bacchus, F., and Yang, Q.
\newblock 1993.
\newblock Downward refinement and the efficiency of hierarchical problem
  solving.
\newblock {\em Artificial Intelligence} 71:43--100.

\bibitem[\protect\citeauthoryear{Backstrom and Nebel}{1996}]{backstrom96}
Backstrom, C., and Nebel, B.
\newblock 1996.
\newblock Complexity results for sas+ planning.
\newblock {\em Computational Intelligence} 11:625--655.

\bibitem[\protect\citeauthoryear{Brafman and Domshlak}{2008}]{brafman2008}
Brafman, R.~I., and Domshlak, C.
\newblock 2008.
\newblock {From One to Many: Planning for Loosely Coupled Multi-Agent Systems}.
\newblock In {\em Proceedings of the 18th International Conference on Automated
  Planning and Scheduling},  28--35.
\newblock AAAI Press.

\bibitem[\protect\citeauthoryear{Brafman and Domshlak}{2013}]{brafman2013}
Brafman, R.~I., and Domshlak, C.
\newblock 2013.
\newblock On the complexity of planning for agent teams and its implications
  for single agent planning.
\newblock {\em Artificial Intelligence} 198(0):52 -- 71.

\bibitem[\protect\citeauthoryear{Brafman}{2006}]{brafman06}
Brafman, R.~I.
\newblock 2006.
\newblock Factored planning: How, when, and when not.
\newblock In {\em Proceedings of the 21st National Conference on Artificial
  Intelligence},  809--814.

\bibitem[\protect\citeauthoryear{Bylander}{1991}]{Bylander1991}
Bylander, T.
\newblock 1991.
\newblock Complexity results for planning.
\newblock In {\em Proceedings of the 12th International Joint Conference on
  Artificial Intelligence}, volume~1,  274--279.

\bibitem[\protect\citeauthoryear{Decker and Lesser}{1992}]{Decker92}
Decker, K.~S., and Lesser, V.~R.
\newblock 1992.
\newblock Generalizing the partial global planning algorithm.
\newblock {\em International Journal of Cooperative Information Systems}
  1:319--346.

\bibitem[\protect\citeauthoryear{Durfee and Lesser}{1991}]{Durfee91}
Durfee, E., and Lesser, V.~R.
\newblock 1991.
\newblock Partial global planning: A coordination framework for distributed
  hypothesis formation.
\newblock {\em IEEE Transactions on Systems, Man, and Cybernetics}
  21:1167--1183.

\bibitem[\protect\citeauthoryear{Helmert}{2006}]{helmert06}
Helmert, M.
\newblock 2006.
\newblock The fast downward planning system.
\newblock {\em Journal of Artificial Intelligence Research} 26:191--246.

\bibitem[\protect\citeauthoryear{Jonsson and Rovatsos}{2011}]{jonsson2011}
Jonsson, A., and Rovatsos, M.
\newblock 2011.
\newblock {S}caling {U}p {M}ultiagent {P}lanning: {A} {B}est-{R}esponse
  {A}pproach.
\newblock In {\em Proceedings of the 21th International Conference on Automated
  Planning and Scheduling},  114--121.
\newblock AAAI Press.

\bibitem[\protect\citeauthoryear{Knoblock}{1994}]{Knoblock94}
Knoblock, C.
\newblock 1994.
\newblock Automatically generating abstractions for planning.
\newblock {\em Artificial Intelligence} 68:243--302.

\bibitem[\protect\citeauthoryear{Kvarnstrom}{2011}]{kvarnstrom11}
Kvarnstrom, J.
\newblock 2011.
\newblock Planning for loosely coupled agents using partial order
  forward-chaining.
\newblock In {\em Proceedings of the 21th International Conference on Automated
  Planning and Scheduling}.

\bibitem[\protect\citeauthoryear{Nissim and Brafman}{2012}]{Nissim2012}
Nissim, R., and Brafman, R.~I.
\newblock 2012.
\newblock Multi-agent a* for parallel and distributed systems.
\newblock In {\em Proceedings of the 11th International Conference on
  Autonomous Agents and Multiagent Systems}, volume~3,  1265--1266.

\bibitem[\protect\citeauthoryear{Nissim, Brafman, and
  Domshlak}{2010}]{Nissim2010}
Nissim, R.; Brafman, R.~I.; and Domshlak, C.
\newblock 2010.
\newblock A general, fully distributed multi-agent planning algorithm.
\newblock In {\em Proceedings of the 11th International Conference on
  Autonomous Agents and Multiagent Systems},  1323--1330.

\bibitem[\protect\citeauthoryear{Torreno, Onaindia, and
  Sapena}{2012}]{torreno2012}
Torreno, A.; Onaindia, E.; and Sapena, O.
\newblock 2012.
\newblock An approach to multi-agent planning with incomplete information.
\newblock In {\em European Conference on Artificial Intelligence}, volume 242,
  762--767.

\end{thebibliography}
\bibliographystyle{aaai}

\end{document}